\documentclass{article}
\usepackage[utf8]{inputenc}
\usepackage{amsfonts}
\usepackage{titling}
\usepackage{amsthm}
\usepackage{amssymb}
\usepackage{hyperref}
\usepackage{amsmath}
\usepackage{xcolor}
\usepackage{mathtools}
\usepackage{comment}
\usepackage{listings}
\usepackage{bbm}
\usepackage{caption}
\usepackage{subcaption}
\usepackage{comment}
\usepackage{graphicx}

\usepackage{titlesec}
\titleformat*{\section}{\large\bfseries}
\titleformat*{\subsection}{\bfseries}

\DeclareMathOperator{\sign}{sign}

\DeclareMathOperator*{\argmin}{arg\,min}
\DeclareMathOperator*{\arginf}{arg\,inf}
\DeclareMathOperator{\diam}{diam}
\newcommand{\indep}{\perp \!\!\! \perp}
\newcommand{\bs}{\boldsymbol}
\newcommand{\mc}{\mathcal}

\usepackage[margin=1.2in]{geometry}
\usepackage{graphicx}

\newtheorem{theorem}{Theorem}
\newtheorem{corollary}{Corollary}

\title{Approximation and generalization properties of the random projection classification method
}
\date{}
\author{Evzenie Coupkova \thanks{ecoupkov@purdue.edu}, Mireille Boutin \thanks{mboutin@purdue.edu}}
\begin{document}
\maketitle
\vspace{-0.45in}

\begin{abstract}
The generalization gap of a classifier is related to the complexity of the set of functions among which the classifier is chosen. We study a family of low-complexity classifiers consisting of thresholding a random one-dimensional feature. The feature is obtained by projecting the data on a random line after embedding it into a higher-dimensional space parametrized by monomials of order up to $k$. More specifically, the extended data is projected n-times and the best classifier among those $n$, based on its performance on training data, is chosen. We show that this type of classifier is extremely flexible as, given full knowledge of the class conditional densities, under mild conditions, the error of these classifiers would converge to the optimal (Bayes) error as $k$ and $n$ go to infinity. We also bound the generalization gap of the random classifiers. In general, these bounds are better than those for any classifier with VC dimension greater than $O(\ln n)$. In particular, the bounds imply that, unless the number of projections $n$ is extremely large, the generalization gap of the random projection approach is significantly smaller than that of a linear classifier in the extended space. Thus, for certain classification problems (e.g., those with a large Rashomon ratio), there is a potentially large gain in generalization properties by selecting parameters at random, rather than selecting the best one amongst the class.
\end{abstract}

\section{Introduction}
Consider a two-class classification problem with real-valued feature vectors, where the dimension of the feature vectors is potentially very high. We seek to use training data to construct a classifier.  
This paper is concerned with a family of extremely simple classifiers. Specifically, we analyze a classification method consisting of projecting the data on a random line so to obtain one-dimensional data.
The one-dimensional data is then classified by thresholding, using the training data to determine the best threshold. This is performed $n$ times so to obtain $n$ different classifiers. The best classifier among those $n$, based on performance on the training data, is then chosen. This yields an affine classifier, whose decision boundary is a hyperplane in the original high-dimensional space. More generally, the data can first be expanded to a higher dimensional space, by concatenating the initial feature vector with monomials in these features, before being projected. By considering all monomials up to some order $k\in {\mathbb N}$, the random projection and thresholding procedure yield a non-affine classifier whose decision boundary is the zero set of a polynomial of order $k$ in the original feature space. We call this classification method {\em thresholding after random projection}. It can be viewed as a one-layer neural network whose parameters in the first layer are chosen at random and whose activation function is a hard-threshold sign function, with the number of different projections $n$ corresponding to the width of the hidden layer.

Considering the simplicity of the thresholding after random projection classification meth\-od, Occam's razor principle suggests that such a classifier should be used for any training dataset that can be well classified after a random projection, as one expects the resulting classifier to generalize well. The first part of this paper quantifies the reason why, by showing how the simplicity of this classification method is related to a low probability of classification error. Specifically, Theorem \ref{ntarpgeneralizationterm} provides an upper bound on the likely difference between the training error and the population error, in terms of the size of the training set $N$ and the number of projections $n$. This bound is expressed independently of the space dimension, independently of $k$, and is lower than that for a non-random linear classifier in the dimension of the extended feature space. In Corollary \ref{bound_on_expectation} we provide a similar bound for the expectation of the magnitude of the generalization gap and in Theorem \ref{our_chaining_result} we strengthen our results by applying chaining technique. For an extremely large number of samples $N$ this bound compares very favorably to the bound for any family of classifiers with a VC dimension larger than $O(\ln(n))$. For reasonably large $N$, the classifier we consider has better generalization properties than ones with VC dimension of order $O(\ln(n)/\ln(N))$. 

In section \ref{experiments} we supported our theoretical results by a set of experiments. We compared the generalization gap of the method of thresholding after random projection and other classification methods: logistic regression, linear SVM and SVM with Gaussian kernel on a high-dimensional synthetic dataset. We also compared the performance of thresholding after random projection on different classification tasks on a real dataset of images that represent digits.

A simple classification method is of little use if it has a poor classification performance. In the second part of this paper, we show that, even though the thresholding after random projection classification method is extremely simple, it has a great approximation power. Specifically, it is likely to approximate, to an arbitrary precision, any continuous function on a compact set (Theorem \ref{continuous_by_random_polynom}) as well as any Boolean function on a compact set that splits the support into measurable subsets (Theorem \ref{polynomial_case}). Based on this, we show that its accuracy as a classifier is asymptotically optimal:
given full knowledge of the class conditional probability distributions, with $k$ and $n$ large enough, a classifier that is arbitrarily close to the (optimal) Bayes classifier is likely to be obtained (Corollary \ref{corollary}). 
 We quantify how many projections are needed to achieve a certain accuracy in a general case (Theorem \ref{howbign}). 
In Section \ref{restricted_number} we also include an example of a classification problem where the number of random projections needed for obtaining a classifier with high accuracy is small.
The reason for this is that the Rashomon ratio for this problem is large. This phenomenon is explored further in \cite{coupkova2024rashomon}.


\section*{Related works}
\textbf{Random projection for clustering.} Thresholding after random projection on a one-dimensional subspace has previously been used successfully to cluster high-dimensional data \cite{yellamrajuboutin,hanboutin}. Projections onto a one-dimensional subspace have also been used to develop fast approximate algorithms (e.g., \cite{kushilevitz2000efficient}). In {\cite{kruskal, Friedman1974APP, pptree_lee_cook} random projections were optimized to highlight the clustering properties of a dataset in one or two dimensions. More generally, random projection on a subspace has been used as a pre-processing step to decrease the dimension of high-dimensional data. The Johnson-Lindenstrauss Lemma suggests that one can decrease the space dimension to a much lower one by random projection while closely preserving the original pairwise distances of a dataset, with high-probability. This property is used in \cite{indyk_clustering} for example.\\
\textbf{Random projection for classification.}  Concerning the problem of classification, it has been shown that a dataset featuring two classes separated by a large margin has a high-probability of being well separated by a random linear separator \cite{blum2005random}. More generally, certain datasets are likely to be divided into two well-separated subsets after projection on a random line \cite{boutin2024}. Empirical tests performed in \cite{cannings2017randomprojection} highlighted the excellent performance of classification by random ensembles, in which classification is performed by thresholding the average label of several lower-dimensional projections. Other related results include \cite{Rahimi2007RandomFF}, where random features are designed to preserve inner products of the transformed data and the data in the feature space of a specific kernel. Their generalization properties are studied in \cite{random_features_generalization}. In \cite{simple_classification_binary_data} it was shown that classifying the data using only its random binary representation is a successful strategy. Assuming that we have only limited information about the datapoints - in particular given a random set of hyperplanes we know on which side of each hyperplane the datapoint lies - it is still possible to successfully classify. Another work that combines several results of one-dimensional projections is \cite{dci}, where the ultimate goal is finding k-nearest neighbours. \\  
\textbf{Generalization and approximation properties of classification by random projection.}  The topic of section \ref{complexity} is similar to the one in \cite{durrant_kaban}. Our main focus is different though: we concentrate on providing a bound for the population error based on the training error after random projection, while in \cite{durrant_kaban} the baseline is the training error of the optimal classifier in the original space. Unlike the bound in \cite{durrant_kaban}, our bound is useful for the extreme case of one-dimensional projections, without adding sparsity or separability conditions on the data. Our work is closely related to \cite{cannings2017randomprojection}. We study the simplest case of choosing just one projection direction, but do not have complex assumptions on the distribution of the data. In the generalization aspect, we add a chaining argument to improve the generalization gap and in the approximation aspect we study the properties of the classifier after expanding the features in a polynomial sense. We also show, how large the number of projections has to be to guarantee a certain error given the degree of polynomial extension and dimension of the data.\\
\textbf{Complexity in general.}  A structured approach to complexity of algorithms and its effect on generalization gap was introduced by Chervonenkis and Vapnik in \cite{vapnik_vcdim} through a concept of VC dimension and growth function. It stated that the less complex the family of functions is the tighter is the bound between training and generalization errors of an algorithm. A more nuanced approach in \cite{boosting_margin} considers the effect of large margin on generalization error, claiming that in case of AdaBoost, the more complex the set of classifiers, the bigger is the margin which leads to a smaller generalization gap. In modern machine learning a big question is why deep neural networks achieve a similar performance on the test set as on the training set, while having a very large complexity. A lot of work has been devoted to this topic, for example \cite{recht_rethinking}, \cite{Kawaguchi_2022} or \cite{robust_generalization_difficult}.

\section{Preliminaries}
\subsection*{Notation}

Let us consider a binary classification problem on $d$-dimensional data $\bs{x} \in \mathbb{R}^d$. Training is based on a labeled set $\mathcal{S} = \{(\bs{x}_i,y_i)\}_{i=1}^N$ that consists of $N$ sample points, where $\bs{x}_i \in \mathbb{R}^d$ and $y_i\in \{-1,1\}$ and each pair $(\bs{x}_i, y_i)$ is drawn i.i.d. from the distribution $\mathcal{D}$ with probability measure $\nu(\bs{x},y)$.

We use bold font to denote vectors or matrices. For a matrix $\bs{A}$ and for numbers $i,j \in \{1,...,n\}$, $\bs{A}_{ij}$ stands for the entry on $i$-th row and $j$-th column of the matrix $\bs{A}$. We denote the Euclidean norm by $\|.\|_2$.

\subsection*{Description of the method}

The classification method discussed in this work is based on two parameters chosen a priori: the number of projections $n$ and the degree of polynomial extension $k$. Given these parameters and the dimensionality $d$ of data, we proceed as described in the following steps:
\begin{itemize}
   \item[(1)] expand the data using the polynomial extension of degree $k$ resulting in a dataset with dimensionality $\tilde{d}=\binom{d+k}{k}$: $\{\bs{x}_{i}^{(k)}\}_{i=1}^{N}$;
   
    \item[(2)] generate $n$ random projection directions $\bs{a}^{(j)}$, $j=1,\ldots, n$, according to a uniform distribution on a unit hypersphere in $\mathbb{R}^{\tilde{d}}$: $\boldsymbol{a}^{(j)} \sim \text{Unif}\left(S^{\tilde{d}-1}\right)$;
 
    \item[(3)] project the $\tilde{d}$-dimensional dataset $\{\bs{x}_{i}^{(k)}\}_{i=1}^{N}$ onto one dimension using each projection direction $\bs{a}^{(j)}$, thus creating one-dimensional data:
    $\{\tilde{x}_{ij}^{(k)}= \bs{a}^{(j)} \cdot \bs{x}_i^{(k)}\}_{i=1}^{N}$;

    \item[(4)] classify each dataset $\{\tilde{x}_{ij}^{(k)}\}_{i=1}^{N}$ using a threshold $t_j$, measure the accuracy of such classification and choose the projection direction $\bs{a}^{(j)}$ with maximum accuracy. 
\end{itemize}

The family of classifiers given by such method is a random subset of affine classifiers given by 
\begin{align}\label{hypothesis_Set_definition}
    \mc{F}_{n,k} = \{f(\bs{x}) = \sign\left(\bs{a}^{(j)} \cdot \bs{x}^{(k)}+t\right),\bs{a}^{(j)} \in \{\bs{a}^{(1)},..., \bs{a}^{(n)}\}, t\in \mathbb{R}^{N}\}.
\end{align}
The set $\mathcal{F}_{n,k}$ is infinite, due to the unspecified value of the threshold $t$, but it is only one dimension, as the projection directions  $\bs{a}^{(j)}$ are fixed.

\section{Generalization properties}\label{complexity}

In this section we focus on the generalization gap of the method of thresholding after random projection. 
Our first result is Theorem \ref{ntarpgeneralizationterm}, which provides an upper bound on the probability that the absolute value of the difference between the training error and the population error is larger than a given $\varepsilon$. It is obtained by splitting the set of functions from which we choose a classifier into independent subsets. We also utilize the fact that the classification is carried out in one dimension, which greatly reduces the size of the set of all possible partitions. In Corollary \ref{bound_on_expectation} we derive a similar bound for the expectation of the magnitude of the generalization gap. As a result, for quite large values of $n$, we get a much tighter bound on the generalization gap than the one given by the VC dimension of classes of functions with VC dimension as low as $2$. We advance these results in Theorem \ref{our_chaining_result} by applying the chaining technique and are subsequently able to compare the asymptotic behaviour of the bound that we propose with the classical one when number of samples converge to infinity. \par

The task of constructing a classifier can be viewed as choosing a hypothesis from a hypothesis set $\mathcal{F}$, based on a training dataset $\mathcal{D} = \{ (\boldsymbol{x_i},y_i) \}_{i=1}^N$, where $\boldsymbol{x_i}$ contains the real-valued features of point $i$, $y_i$ is the class of point $i$, and $N$ is the number of points. The set $\mathcal{F}$ needs to be rich enough to approximate the optimal solution well, but not too rich, with respect to $N$, or else the generalization error of the chosen classifier may turn out to be much different from the training error. 

The training error is defined in the following way:
\begin{align}
    R_{\text{train}}^{\mathcal{D}}(f) = \frac{1}{N}\sum_{i=1}^N \mathbbm{1}(f(\boldsymbol{x}_i)\neq y_i)
\end{align}
and a common technique of deciding which hypothesis from $\mathcal{F}$ to choose is based on minimizing the training error (the Empirical Risk Minimization method or ERM). The method of thresholding after random projection also uses this approach and the final classifier $\hat{f}$ is $\hat{f} = \argmin_{f\in \mathcal{F}} R_{\text{train}}^{\mathcal{D}}(f)$. The ultimate goal, however, is to find a classification that will be accurate on a new data set. In other words, one would like to accurately predict the class of data points that do not belong to the training set. Ideally, the error of the classifier on a new data set would be likely to be close to the training error. But in actuality, it might likely be much larger. If we assume that our data are generated by a probability distribution $\rho_{\boldsymbol{X},Y}(\boldsymbol{x},y)$ on a certain space $\mathcal{E} \times \{-1,1\}$, we can compute the overall {\em population error} of the classifier, that is to say the error that the classifier would make if it were used to classify all of the points in $\mathcal{E}\times \{-1,1\}$:
\begin{align}\label{population_error}
    R_{\text{popul}}(f) &= 
    \int_{\mathcal{E}\times\{-1,1\}}\rho_{\boldsymbol{X}, {Y}}(\boldsymbol{x}, y)\mathbbm{1}\left(f(\boldsymbol{x})\neq y\right)d(\boldsymbol{x},y).
\end{align}
We want to guarantee that applying the classifier on a new data set will likely yield an error similar to the training error. More specifically, we want that, for any given level of tolerance $\delta \in (0,1)$, with probability at least $1-\delta$, for any function $f\in \mathcal{F}$ the difference between the population error and the training error is less than some small generalization term $R_{\mathcal{F}}$: 
\begin{align*}
    P\left(\sup_{f \in \mathcal{F}}|R_{\text{popul}}(f) - R_{\text{train}}^{\mathcal{D}}(f)| \leq R_{\mathcal{F}}\right) >1-\delta.
\end{align*}
We call the quantity 
\begin{align}\label{generalization_gap}
    \sup_{f \in \mathcal{F}}|R_{\text{popul}}(f)-R_{\text{train}}^{\mathcal{D}}(f)|
\end{align}
generalization gap that corresponds to the family of classifiers $\mathcal{F}$. The generalization term $R_{\mathcal{F}}$ may be expressed as a function that depends on the number of points in the training set $N$, the tolerance level $\delta$ and the richness of the family of functions $\mathcal{F}$: $R_{\mathcal{F}}=R_{\mathcal{F}}(\delta, N)$.

When our task is binary classification, the richness of the family $\mathcal{F}$ can be measured by a function that computes how many different outcomes can be achieved on dataset of size $N$ if classifiers from $\mathcal{F}$ are used. This function is called a growth function and is denoted by $m_{\mathcal{F}}(N)$, where $    m_{\mathcal{F}}(N) = \max_{\boldsymbol{x}_1, ..., \boldsymbol{x}_N \in \mathbb{R}^d}|\{(f(\boldsymbol{x}_1), ..., f(\boldsymbol{x}_N)), f\in \mathcal{F}\}|$. Since functions $f\in\mathcal{F}$ are Boolean, the value of $m_{\mathcal{F}}(N)$ is upper bounded by $2^N$ for each $N$. We call the vectors $(f(\boldsymbol{x}_1), ...,f(\boldsymbol{x}_N)))$ dichotomies and denote by $\mathcal{A}_{\mathcal{F}}(\boldsymbol{x}_1,...,\boldsymbol{x}_N)=\mathcal{A}_{\mathcal{F}}(\boldsymbol{x}^N)$ the set of dichotomies that a class $\mathcal{F}$ is able to produce on the set of points $\boldsymbol{x}_1, .., \boldsymbol{x}_N$. Using this notation we can also define growth function as
\begin{align*}
    m_{\mathcal{F}}(N) = \max_{\boldsymbol{x}_1, ..., \boldsymbol{x}_N \in \mathbb{R}^d}|\mathcal{A}_{\mathcal{F}}(\boldsymbol{x}_1, ..., \boldsymbol{x}_N)\}|.
\end{align*}
In terms of growth function, we can bound the generalization gap of $\mathcal{F}$ by $R_{\mathcal{F}}(\delta, N) = \sqrt{\frac{8}{N}\ln\left(\frac{4m_{\mathcal{F}}(2N)}{\delta}\right)}$.
If the set $\mathcal{F}$ has a finite VC dimension $d_{VC}$, according to Sauer-Shelah lemma, its growth function is bound by a polynomial in $N$:
\begin{align*}
m_{\mathcal{F}}(N) \leq \sum_{i=0}^{d_{VC}} \binom{N}{i} &\leq  \left(\frac{Ne}{d_{VC}}\right)^{d_{VC}}.
\end{align*}
Replacing the growth function by these quantities one gets
\begin{align}\label{VC_ineq_tight}
      R_{\mathcal{F}}(\delta, N)\leq \sqrt{\frac{8}{N}\ln\left({\frac{4\left(\frac{2eN}{d_{VC}}\right)^{d_{VC}}}{\delta}}\right)}.
\end{align} 
In particular, the VC-dimension of the class of affine functions $\mathcal{F}_{\text{affine}}$ applied to classify points in $\mathbb{R}^d$ is $d+1$ \cite{AbuMostafa2012LearningFD}. Therefore, we have
\begin{align}\label{vc_ineq_ntarp}
 R_{\mathcal{F_\text{affine}}}(\delta, N) \leq \sqrt{\frac{8}{N}\ln\left({\frac{4\left(\frac{2eN}{d+1}\right)^{d+1}}{\delta}}\right)}.
\end{align}

Let us consider the family of classifiers that corresponds to the thresholding after random projection method that uses $n\in \mathbb{N}$ projections after expanding the feature space by a polynomial transformation of order $k\in \mathbb{N}$ denoted by $\mathcal{F}_{n,k}$. The class $\mathcal{F}_{n,k}$ depends on the projection directions that are generated in the first step, hence it is a random set of functions. Its growth function $m_{\mathcal{F}_{n,k}}(N)$ can be upper bounded by a quantity that depends on the number of projections only. To prove the bound, we do not use randomness of the classifier, merely the fact that the set of all projection directions considered is limited to $n$. The randomness is used in Section \ref{optimality}, where it guarantees density of the family of classifiers $\mathcal{F}_{n,k}$ in the set of Bayes classifiers that split the support into two measurable sets. 
 We are going to use this bound in the proof of the following Theorem \ref{ntarpgeneralizationterm}. The results in this section hold under general conditions, we do not pose any assumptions on the distribution of the data or on the configuration of the training set.

 In the following theorems there are two levels of randomness that we need to consider. The first one comes from the random pick of projection direction. The second one comes from the random pick of data. The following argument shows that one level of randomness that stems from the random choice of the projection directions does not affect our bounds. This is due to the fact that the same inequalities hold for any random pick of $n$ projection directions $(\bs{a}^{(1)},...,\bs{a}^{(n)})$. Let us denote the $n$-tuple of projection directions $(\bs{a}^{(1)},...,\bs{a}^{(n)})$ by $\bs{a}_d^{(n)}$, then we have:
\begin{align}\label{condition}
\begin{split}
     P\left(\sup_{f \in \mathcal{F}_{n,k}}\vert R_{\text{popul}}(f)- R_{\text{train}}^{\mathcal{D}}(f) \vert \leq \Delta\right) =\\
     =\int_{\bs{a}^{(1)}\in \mc{S}^{d-1}}...\int_{\bs{a}^{(n)}\in \mc{S}^{d-1}} P\left(\sup_{f \in \mathcal{F}_{n,k}}|R_{\text{popul}}(f)- R_{\text{train}}^{\mathcal{D}}(f)| \leq \Delta \mid  \bs{a}_d^{(n)}\right) P\left(\bs{a}_d^{(n)}\right)d\left(\bs{a}_d^{(n)}\right)
\end{split}
\end{align}
and if we know that 
\begin{align*}
     P\left(\sup_{f \in \mathcal{F}_{n,k}}|R_{\text{popul}}(f)- R_{\text{train}}^{\mathcal{D}}(f)| \leq \Delta \mid  \bs{a}_d^{(n)}\right) \geq \beta
\end{align*}
we can conclude that the left hand side of \ref{condition} is larger than $\beta$ as well:
\begin{align*}
     P\left(\sup_{f \in \mathcal{F}_{n,k}}\vert R_{\text{popul}}(f)- R_{\text{train}}^{\mathcal{D}}(f) \vert \leq \Delta\right)\geq \beta\int_{\bs{a}^{(1)}\in \mc{S}^{d-1}}...\int_{\bs{a}^{(n)}\in \mc{S}^{d-1}}P\left(\bs{a}_d^{(n)}\right)d\left(\bs{a}_d^{(n)}\right) = \beta.
\end{align*}
 
\begin{theorem}\label{ntarpgeneralizationterm}
Let us consider a family $\mathcal{F}_{n,k}$ of classifiers that correspond to the method of thresholding after random projection. Set the tolerance level $\delta>0$. Then we have that
\begin{align}\label{ntarp_best_bound}
    P\left(\sup_{f \in \mathcal{F}_{n,k}}|R_{\text{popul}}(f)- R_{\text{train}}^{\mathcal{D}}(f)| \leq \sqrt{\frac{8}{N}\ln\left({\frac{16\;n\;N}{\delta}}\right)}\right) \geq 1-\delta.
\end{align}
Notice that the bound on the generalization gap does not depend on the order of extension $k$.
\end{theorem}
\begin{proof}
 The hypothesis set of the method of thresholding after random projection with $n$ iterations can be split into $n$ independent subsets:
\begin{align*}
    \mathcal{F}_{n,k} = \bigcup_{i=1}^n \mathcal{F}_{n,k}^{(i)},
\end{align*}
here each $\mathcal{F}_{n,k}^{(i)}$ is the family of classifiers that we choose from after $i-$th projection. If $\boldsymbol{a}_i$ is the $i$-th projection direction then $\mathcal{F}_{n,k}^{(i)} = \{\sign\left(\sigma \boldsymbol{a}_i\cdot\phi_k(\boldsymbol{x})+\tau\right), \tau \in \mathbb{R}, \sigma \in \{-1,1\}\}$.

We now count how many different outcomes (i.e., dichotomies) are possible if we use the hypothesis set $\mathcal{F}_{n,k}$ on the dataset of size $N$. The set $\mathcal{F}_{n,k}$ depends on a random choice of the projection directions $\{\boldsymbol{a}_1,...,\boldsymbol{a}_n\}$, but for each such $n$-tuple of vectors, the same upper bound on the number of different dichotomies holds. When data are projected in some direction, they are arranged in a certain order on a line. Then a threshold is chosen in one dimension that separates the points into two classes: to the left of the threshold and to the right of the threshold. We can choose which class is on the left, which class is on the right.

Thus, each $\mathcal{F}_{n,k}^{(i)}$ gives us $2N$ different possible dichotomies on $N$ points. That corresponds to projection in one particular random direction. If we choose a different projection direction, the order of the points might become different, which gives us different dichotomies. At most, we can get $2N$ more dichotomies with each projection. Therefore the number of different outcomes for $N$ points is no more than $2nN$. Here we take advantage of the fact that Sauer-Shelah lemma (see \cite{sauer} for example) is not tight for classification by linear separation in one dimension, which has VC dimension equal to $2$. That is due to the fact that for $N>2$
\begin{align*}
    \sum_{i=0}^2 \binom{N}{i} > 2N.
\end{align*}

From \cite{vapnik71uniform} and \cite{AbuMostafa2012LearningFD} we know that 
\begin{align*}
    \sup_{f \in \mathcal{F}}|R_{\text{popul}}(f)- R_{\text{train}}^{\mathcal{D}}(f)| \leq \sqrt{\frac{8}{N}\ln\left({\frac{4m_{\mathcal{F}}(2N)}{\delta}}\right)}.
\end{align*}
Since $m_{\mathcal{F}}(2N) \leq 4nN$, we get the desired result.
\end{proof}
While the upper bound in (\ref{first_bound}) converges to infinity with $n$ increasing, there exists a better bound for excessively large $n$. That is due to the fact that the class of functions given by $n$ random projections is not richer than the class of affine functions used as classifiers whose VC dimension is $d+1$ (where $d$ is the dimension of the space, in which we generate the random projection directions). Therefore we get 
\begin{align*}
\sup_{f \in \mathcal{F}}|R_{\text{popul}}(f)- R_{\text{train}}^{\mathcal{D}}(f)| \leq \sqrt{\frac{8}{N}\ln\left({\frac{4\min\Big\{4nN,\left(\frac{2Ne}{d+1}\right)^{(d+1)} \Big\}}{\delta}}\right)}
\end{align*}
And if
\begin{align}\label{too_big_n}
    n > \frac{1}{4N}\left(\frac{2Ne}{d+1}\right)^{d+1},
\end{align}
the classical bound for the generalization gap of affine functions is tighter than the bound that we propose. Number of projections $n$ has to be very large to overcome this bound though. For illustration, if the number of samples is equal to $1000$ then the bound that $n$ has to overcome is 
$1800$ for $d=2$, $1.5\times10^6$ for $d=3$ and $3.8\times10^{11}$ for $d=5$. With larger number of samples, this threshold grows as well as with the larger number of dimensions.

\subsection{Bound on the expected value of the generalization gap}
A similar bound is true for the expectation of the absolute value of the difference between the training and population errors. In this case, the parameter $\delta$ is not present, our bound depends only on the sample size $N$ and number of random projections $n$.

\begin{corollary}\label{bound_on_expectation}
Let us consider a family  $\mathcal{F}_{n,k}$ of classifiers that correspond to the method of thresholding after random projection. We have that
\begin{align*}
E\left(\sup_{f\in \mathcal{F}_{n,k}}|R_{\text{train}}^{\mathcal{D}}(f)-R_{\text{popul}}(f)|\right) \leq \sqrt{\frac{2\ln(8nN)}{N}}.
\end{align*}
\end{corollary}
\begin{proof}
A Theorem 1.9 from \cite{Lugosi2002} gives 
\begin{align*}
E\left(\sup_{f\in \mathcal{F}}|R_{\text{train}}^{\mathcal{D}}(f)-R_{\text{popul}}(f)|\right) \leq \sqrt{\frac{2\ln(2m_{\mathcal{F}}(2N))}{N}},
\end{align*}
where $m_{\mathcal{F}}(2N)$ is a growth function of a class of functions $\mathcal{F}$ on $2N$ points. Since for the class of functions given by classification after random projection $\mathcal{F}_{n,k}$ the growth function on $2N$ points is always less or equal than $4nN$ we get the stated result. As in the previous Theorem, if $n$ is excessively large, the classical bound that uses VC dimension of affine functions is tighter than the one we propose. The magnitude of such $n$ can be calculated using the same formula: (\ref{too_big_n}).
\end{proof}

\subsection{Illustration of advantage in generalization gap for various numbers of projections}

 Unlike the VC dimension estimate on the generalization gap, the bound that we propose does not depend on the VC dimension, but on the number of projections $n$. In Figure \ref{fig:comparisonwithdvc} we compare the values of the bound for the expected value of the generalization gap for the method of thresholding after random projection for $n$ varying between $1$ and $1000$ with the estimate given by a VC dimension for a method with $d_{VC}$ equal to $2$ and $3$. Our estimate grows with the number of projection directions applied, but the growth is logarithmic and for all values of $n$ considered in the graph and the thresholding after random projection method is better by nearly half/at least one and a half percentage points, respectively.

\begin{figure}[ht!]
    \centering
    \includegraphics[scale=0.7]{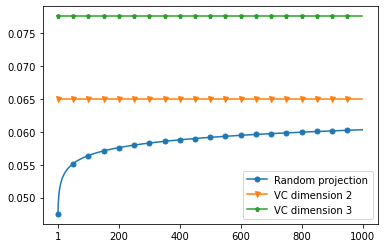}
    \caption{A comparison between the bounds for the expectation of the generalization gap of the method of thresholding after $n$ random projections and an algorithm with VC dimension $d_{VC}=2$ and $d_{VC}=3$, respectively. The graph is generated under the assumption that $N=10000$.}
    \label{fig:comparisonwithdvc}
\end{figure}

\subsection{Comparison with linear separation when dimensionality increases}
If the feature data is $d$-dimensional, one can build a classifier by picking the {\em best} hyperplane following some goodness-of-fit criterion. As stated earlier, the VC dimension of this classification method is $d+1$. This is one of the simplest classification methods. Yet, for the method of thresholding after random projection, we have provided a bound on the generalization gap that is smaller than that for a classification method with $d_{VC}=d+1$ for any $d\geq 1$ as long as $n$ is not too large. 

To illustrate the difference, let us compare the generalization term for the method of thresholding after random projection given by our estimate and the generalization term given by an estimate that uses VC dimension $d_{VC}=d+1$. We assume that we do not need to extend the feature space ($k=1$) and fix $n$ to reach the optimal training error as discussed in Section \ref{optimality} as per Formula (\ref{howlargen}). However, since the bound on $n$ is overly conservative, the optimal error might be achieved with a smaller $n$. As we can see from Table \ref{linearseparationtable}, the generalization advantage of the method of thresholding after random projection is present even for 2-dimensional data. For data in 10 dimensions, the estimate for the generalization gap of the linear separation algorithm is almost twice as large as the estimate for the generalization gap of the method of thresholding after random projection.

\begin{table}[ht!]
    \centering
    \begin{tabular}{c c c c c}
        $d$ & $n$ & $R_{\mathcal{F}_{n,1}}$ & $R_{\mathcal{F}_{\text{affine}}}(d+1)$ \\
         $2$ & $25$ & $0.118$ & $0.163$\\
         $3$ & $117$ & $0.123$ &  $0.183$\\
         $4$ & $592$ & $0.129$ &  $0.200$\\
         $5$ & $3278$ & $0.134$ &  $0.216$\\
         $6$ & $19664$ & $0.139$ &  $0.230$ \\
         $7$ & $126414$ & $0.144$ & $0.244$ \\
         $8$ & $863981$ & $0.150$ & $0.256$ \\
         $9$ & $6236483$ & $0.155$  &$0.268$\\
         $10$ & $47292177$ & $0.160$ &$0.279$
    \end{tabular}
    \caption{Comparing the bounds on the generalization gaps, $R_{\mathcal{F}_{n,1}}$ for the method of thresholding after random projection that uses $n$ projections and $R_{\mathcal{F}_{\text{affine}}}(d+1)$ for linear classification that has VC dimension of $d+1$. We set $n$ to be greater than the bound stated in Theorem \ref{howbign} using 
     $\delta=0.1$ and $N=10000$, and we assume that $k=1$ is enough for a good classification result.}
    \label{linearseparationtable} 
\end{table}

\subsection{Application of chaining technique}
There exists another bound on the expectation of the generalization gap that is better than the one discussed previously for very large $N$ (therefore it is better asymptotically, when $N$ converges to infinity). The bound is obtained using chaining technique (described in \cite{Dudley_chaining}) and eradicates the logarithmic term in the number of samples from the numerator:
\begin{align}\label{classical_chaining_bound}
E\left(\sup_{f \in \mathcal{F}}|R_{\text{popul}}(f)-R_{\text{train}}^{\mathcal{D}}(f)|\right) \leq 65.16\sqrt{\frac{d_{VC}}{N}},    
\end{align}
where $d_{VC}$ is the VC dimension of the class $\mathcal{F}$. In this section we prove a similar result for the generalization gap of the method of thresholding after random projection where we roughly speaking replace the VC dimension by $\ln(n)$ term. So far we used the exact number of different dichotomies that the method of random projections applied $n$ times can produce on a fixed number of datapoints. There is another property of the set of dichotomies given by the method of random projections $\mathcal{A}_{\mathcal{F}_{n,k}}(\boldsymbol{x}^N)$ which we have not used yet. This property allows for a different bound on the generalization gap, where we also eliminate the logarithmic term in the number of samples from the numerator. Our class of functions is not only small in magnitude, it also has a simple geometric structure that results in low covering numbers. Classification after each projection results in dichotomies that lie in a chain on a hypercube - which means that there exists an ordering of dichotomies, such that the Hamming distance between consecutive dichotomies is equal to one. 

\begin{theorem}\label{our_chaining_result}
Let us consider a family $\mathcal{F}_{n,k}$ of classifiers that correspond to the method of thresholding after random projection. The expectation of the generalization gap corresponding to $\mathcal{F}_{n,k}$ is bounded by 
\begin{align} \label{chaining_bound}
    E\left(\sup_{f\in \mathcal{F}_{n,k}}|R_{\text{train}}^{\mathcal{D}}(f)-R_{\text{popul}}(f)|\right)  \leq \frac{24}{\sqrt{N}} \left(\sqrt{\ln(n)}+1.66\right).
\end{align}
\end{theorem}
\begin{proof}
Theorem 1.16 from \cite{Lugosi2002} says that 
\begin{align*}
    E\left(\sup_{f\in \mathcal{F}}|R_{\text{train}}^{\mathcal{D}}(f)-R_{\text{popul}}(f)|\right)  \leq \frac{24}{\sqrt{N}}\max_{x_1,...x_N \in \mathbb{R}^d}\int_0^1\sqrt{\ln\left(2\mathcal{N}(r,\mathcal{A}(x^N))\right)}dr, 
\end{align*}
where $\mathcal{N}(r,\mathcal{A}(x^n))$ is a covering number of the set of dichotomies on $N$ datapoints $\mathcal{A}(\boldsymbol{x}^N)$ with radius $r$. The covering number is computed with respect to the square root of the normalized Hamming distance:
\begin{align*}
    d(\boldsymbol{x},\boldsymbol{y}) = \sqrt{\frac{1}{N}\sum_{i=1}^{N}\mathbbm{1}(x_i \neq y_i)}.
\end{align*}
We would like to upper bound covering numbers $\mathcal{N}(r, \mathcal{A}_{\mathcal{F}_{n,k}}(\boldsymbol{x}^N))$ in a particular case of the method of thresholding after random projection. If we use only one projection (if $n=1$) then the set of dichotomies on $N$ datapoints has no more than $2N$ elements:
\begin{align*}
    \max_{x_1, ..., x_N \in \mathbb{R}^d} |\mathcal{A}_{\mathcal{F}_{n,k}}(x^N)| \leq 2N.
\end{align*}
The dichotomies in $\mathcal{A}(\boldsymbol{x}^N)$ create a chain: there exists an ordering of its elements, such that the Hamming distance between neighbors is equal to one. That means that $d(\boldsymbol{x}, \boldsymbol{y}) = \sqrt{{1/N}}$ for the neighboring dichotomies $\boldsymbol{x}$ and $\boldsymbol{y}$. Therefore, if radius $r$ belongs to the interval $[0,\sqrt{1/N})$, we can cover the chain of dichotomies using all of its elements:
\begin{align*}
    \mathcal{N}(r,\mathcal{A}_{\mathcal{F}_{n,k}}(\boldsymbol{x}^N)) = |\mathcal{A}_{\mathcal{F}_{n,k}}(\boldsymbol{x}^N)| \leq 2N,
\end{align*}
and if $r$ is slightly larger and lies in the interval $[\sqrt{1/N}, \sqrt{2/N})$ then taking one point out of three is enough to cover the whole chain:
\begin{align*}
     \mathcal{N}(r,\mathcal{A}_{\mathcal{F}_{n,k}}(x^N)) = \left\lceil\frac{|\mathcal{A}_{\mathcal{F}_{n,k}}(x^N)|}{3}\right\rceil \leq \left\lceil\frac{2N}{3}\right\rceil.
\end{align*}
In general, if $r \in [\sqrt{(i-1)/N}, \sqrt{i/N})$ for $1\leq i \leq N$, then 
\begin{align*}
     \mathcal{N}(r,\mathcal{A}_{\mathcal{F}_{n,k}}(x^N)) = \left\lceil\frac{|\mathcal{A}_{\mathcal{F}_{n,k}}(x^N)|}{2i+1}\right\rceil \leq \left\lceil\frac{2N}{2i+1}\right\rceil.
\end{align*}
Therefore, we can bound the integral from Theorem 1.16 in the following way:
\begin{align*}
  \max_{\boldsymbol{x}_1,..., \boldsymbol{x}_N \in \mathbb{R}^d}\int_0^1\sqrt{\ln\left(2\mathcal{N}(r,\mathcal{A}_{\mathcal{F}_{n,k}}(x^N))\right)}dr\leq \sum_{i=1}^{N}\left(\sqrt{\frac{i}{N}}-\sqrt{\frac{i-1}{N}}\right)\sqrt{\ln\left(2\left\lceil\frac{2N}{2i+1}\right\rceil\right)}.
\end{align*}
The sum on the right can be viewed as a lower bound for the Riemann integral that uses $\Big\{\left[\sqrt{\frac{i-1}{N}}, \sqrt{\frac{i}{N}}\right]\Big\}_{i=1}^{N}$ as a partition of interval $[0,1]$. We know that for each $N$ the lower Riemann sum is smaller than the Riemann integral. Hence we get:
\begin{align}\label{integral_is_constant}
    \sum_{i=1}^{N-1}\left(\sqrt{\frac{i}{N}}-\sqrt{\frac{i-1}{N}}\right)\sqrt{\ln\left(2\left\lceil\frac{2N}{2i+1}\right\rceil\right)} \leq \int_0^1 \sqrt{\ln{\left(\frac{2}{r^2}+2\right)}} dr \leq 1.66.
\end{align}
where the inequality is based on a choice of partitioning $    r = \sqrt{\frac{i}{N}} \Rightarrow r^2 = \frac{i}{N}
$ which leads to 
\begin{align*}
    \left\lceil\frac{2N}{2i+1} \right\rceil \leq \left\lceil\frac{N}{i} \right\rceil \leq \frac{N}{i} + 1 = \frac{1}{r^2}+1.
\end{align*}
Using the bound from formula (\ref{integral_is_constant}) and Theorem 1.16 we get that
\begin{align*}
     E\left(\sup_{f \in \mathcal{F}_1}|R_{\text{train}}^{\mathcal{D}}(f)-R_{\text{popul}}(f)|\right) \leq \frac{24}{\sqrt{N}} \int_0^1 \sqrt{\ln{\left(\frac{2}{r^2}+2\right)}} dr \leq \frac{39.84}{\sqrt{N}}.
\end{align*}

We would like to expand this result to $n$ projections. After each projection a chain of dichotomies is created. Different chains might be close to each other on the hypercube or they might be far apart. We can cover each chain separately and obtain an upper bound on the covering number. For each $r>0$ the dichotomies that are the result of $n$ random projections can be bound by $n$ times the size of the covering set for dichotomies after just one projection:
\begin{align*}
    \mathcal{N}(r,\mathcal{A}_{\mathcal{F}_{n,k}}(\boldsymbol{x}^N))\leq n\mathcal{N}(r, \mathcal{A}_{\mathcal{F}_1}(\boldsymbol{x}^N)).
\end{align*}
as a result, we get the following inequality:
\begin{align*}
    E\left(\sup_{f \in \mathcal{F}_{n,k}}|R_{\text{train}}^{\mathcal{D}}(f)-R_{\text{popul}}(f)|\right) \leq \frac{24}{\sqrt{N}} \sum_{i=1}^{N-1}\left(\sqrt{\frac{i}{N}}-\sqrt{\frac{i-1}{N}}\right)\sqrt{\ln\left(2n\left\lceil\frac{2N}{2i+1}\right\rceil\right)}
\end{align*}
Following similar arguments as in the case of one projection, bounding the Riemann sum by an integral we get
\begin{align*}
      E\left(\sup_{f \in \mathcal{F}_{n,k}}|R_{\text{train}}^{\mathcal{D}}(f)-R_{\text{popul}}(f)|\right) \leq \frac{24}{\sqrt{N}} \int_0^1 \sqrt{\ln{\left(n\right)+\ln\left(\frac{2}{r^2}+2\right)}} dr,
\end{align*}
which is less or equal than $\frac{24}{\sqrt{N}}\left(\sqrt{\ln(n)}+1.66\right)$.
\end{proof}

\subsection{Asymptotic comparison to a classifier with a given VC dimension when number of samples converges to infinity}
We now compare the generalization gap of the method of thresholding after random projection with that of a classifier with a given VC dimension. Specifically, we consider the limit of the ratio of their generalization gaps when the number of data points in the training set converges to infinity. The result is in the following theorem.  

\begin{corollary}\label{ratio_improved}
For large enough training sets, the bound (\ref{chaining_bound}) on the generalization gap of the thresholding after random projection classification method is smaller than the bound (\ref{classical_chaining_bound}) on the generalization gap estimated using the VC dimension for any algorithm with VC dimension larger than $\left(\frac{24}{65.16}\left(1.66+\sqrt{\ln(n)}\right)\right)^2$. More specifically, when the number of samples goes to infinity, the ratio of the generalization gaps goes to $\frac{24}{65.16}\left(\sqrt{\frac{\ln(n)}{d_{VC}}}+\frac{1.66}{\sqrt{d_{VC}}}\right)$.
\end{corollary}
\begin{proof}
We have:
\begin{align*}
    \lim_{N \rightarrow \infty} \dfrac{ \frac{24}{\sqrt{N}}\left(\sqrt{\ln(n)}+1.66\right)}{ 65.16\sqrt{\frac{d_{VC}}{N}}} = \frac{24}{65.16}\left(\sqrt{\frac{{\ln(n)}}{{d_{VC}}}}+\frac{1.66}{\sqrt{d_{VC}}}\right).
\end{align*}
\end{proof}
This means that as long as the number of projections $n$ satisfies the following inequality:
\begin{align*}
    n < \exp\left(d_{VC}\left(\frac{65.16}{24}-\frac{1.66}{\sqrt{d_{VC}}}\right)^2\right),
\end{align*}
the bound on the generalization gap of the method of thresholding after $n$ random projections is tighter than the one that corresponds to an algorithm with VC dimension equal to $d_{VC}$. For example, if $d_{VC}$ equals $3$, we can have up to $10476$ random projection directions to achieve a similar bound, while for $d_{VC}=4$, we can have $n$ as large as $1487935$. 

\section*{Experiments}\label{experiments}
\subsubsection*{Synthetic data, model definition}
Let us consider a high-dimensional distribution that is motivated by an example from \cite{boutin2024}. It is a mixture of two marginal distributions that represent two classes: 
$(X,Y) \sim P$, where $X \in \mathbb{R}^d$, with $d$ potentially very large, $Y \in \{-1,1\}$. 
 For $k\in \{-1,1\}$, the distribution of $X$ given $Y=k$ is the following:
\begin{align*}
    X|Y=k \sim (X_1^k, ..., X_d^k),
\end{align*}
where 
\begin{align}\label{model}
    \forall i \neq j: \;\; X_i^k \indep X_j^k \;\; \text{ and } X_i^k \sim \text{Ber}(p_i^k) + \mathcal{N}(0,\sigma^2)
\end{align}
with $\bs{p}^k = (p_1^k,...,p_d^k)$ a vector of Bernoulli parameters.
In other words, the components of the data points are independent Bernoulli random variables, perturbed by Gaussian noise. Thus, each class can be seen as being drawn from a mixture of Gaussians whose means are situated on the vertices of a $d$-dimensional cube. The vectors $\bs{p}^{-1}$ and $\bs{p}^{1}$ determine the likelihood of ``belonging" to each vertex for each class, respectively. For example, if $p_1^{-1}=0$ and $p_1^{1}=1$, then the first component of the data points is drawn from a Gaussian with zero mean for the first class, while that of the second class is drawn from a Gaussian with mean equal to one.  
In the extreme case when the parameters are equal: $\bs{p}^1=\bs{p}^{-1}$, the classes are mixed up to the highest possible extend, and so Bayes error is equal to $50\%$ and no classifier can have a population error better than $50\%$. It has been shown empirically that classification methods tend to overfit on similar datasets. For example in \cite{recht_rethinking} it was shown that large enough neural networks are able to memorize a dataset with random labeling so that the training error was close to $0\%$, the test error was close to $50\%$, which lead to an overwhelming generalization gap.

When the parameters $\bs{p}^{1}$ and $\bs{p}^{-1}$ get further apart, the data become less mixed and are easier to separate. Another extreme case is when $\bs{p}^{1}$ is a vector of ones and $\bs{p}^{-1}$ is a vector of zeros. In this case, each class is drawn from a single Gaussian and Bayes classifier is a linear classifier.

\subsubsection*{Setup of the experiment}
We set the noise $\sigma$ to two different levels: $\sigma=0$ and $\sigma=0.05$. We also consider $20$ different combinations of parameters $\bs{p}^{1}$ and $\bs{p}^{-1}$. We start with an extreme case, when $\bs{p}^{1}=\bs{p}^{-1}=(0.25, ...0.25, 0.8, ..., 0.8,1)$ and then change $\bs{p}^{-1}$ linearly to $(0.8, ...0.8, 0.25, ..., 0.25, 1)$ to see how the gap between training and test errors changes for different levels of mixture of classes. We fix the dimension to $d=65$, the number of training points to $200$, with half belonging to the first class and the other half belonging to the second class (priors are equal). We generate an independent test dataset that follows the same distribution and consists of $2000$ points with equal priors. With this size of the test set, Hoeffding inequality tells us that with probability $0.9$, the difference between the empirical test error and the true-population error will be less than $2.8\%$ regardless of the data distribution. Indeed we have
\begin{align*}
    P\left(|\text{test error}(f)- \text{population error}(f)|\leq \sqrt{\frac{\ln(2/\delta)}{2N}}\right)\geq \delta,
\end{align*}
where $f$ is the classifier being tested.\\
We train four models on the training set and compute their empirical generalization gap - the difference between test and training errors. We repeat the process of generating training and testing data $5$ times, retrain the four models each time and compute an average for the generalization gap as well as its standard deviation.  Finally, we plot the average and standard deviation of the generalization gap; the x-axis is labeled with the steps we have steps that represent the combination of parameters $\bs{p}^{1}$ and $\bs{p}^{-1}$ (see Figure \ref{gap}). 
\begin{figure}
\begin{minipage}{.5\linewidth}
\centering
{\label{zero_sigma}\includegraphics[scale=.55]{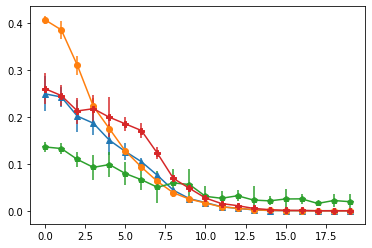}}
\end{minipage}%
\begin{minipage}{.5\linewidth}
\centering
{\label{fivesigma}\includegraphics[scale=.55]{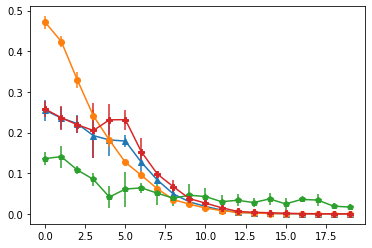}}
\end{minipage}\par\medskip
\caption{Empirical generalization gap for different mixture levels of data. Parameters are: $d=65$, $\bs{p}^{1}=(0.25,...,0.8,1)$, at $x=20$ $\bs{p}^{-1}=(0.8,...,0.25,1)$, $N_{\text{train}}=200$, $N_{\text{test}}=2000$, green line with pentagons correspond to thresholding after random projection with $n=10000$, blue line with triangles corresponds to logistic regression, red line with crosses corresponds to linear SVM and orange line with circles corresponds to SVM with Gaussian kernel. The left and right figures differ in the level of noise: (a) $\sigma=0$ (b) $\sigma=0.05$. The theoretical bounds for the generalization gap are 2.7 for the linear models (logistic regression and SVM) and  0.8 for thresholding after random projection). Note that the distribution of the data represented on the horizontal axis differs only in the labels - the distribution of independent variables $\bs{x}$ stays the same. Therefore this experiment shows amongst other things that the generalization bounds should take labels into account even for simple models (similar issue for the case of deep models is discussed in \cite{recht_rethinking}).}
\label{gap}
\end{figure}

\subsubsection*{Discussion of the results}
A smaller generalization gap is a certificate of robustness of an algorithm; better guarantees for generalization mean that we can put more trust in the training error. There are certain setups in learning where a test dataset is not available, which implies that the only indicator of the performance is the training error, to be interpreted within the context of the theoretical bounds on the generalization gap. In our experimental setup, the bound on the generalization gap for thresholding after random projection is $80\%$, while that of the other linear classifiers, logistic regression and linear SVM it is $270\%$ and that of the Gaussian kernel SVM is infinite when an arbitrary number of support vectors is used. While the bounds are probabilistic, and not necessarily tight, our experiments show a similar trend and thus support our theoretical results. Indeed according to Figure \ref{gap} the empirical gap for thresholding after random projection is never larger than 15\%, while for logistic regression and linear SVM it can be $25\%$. SVM with Gaussian kernel can overfit even more: in some cases its generalization gap is larger than $40\%$. \\
In our experimental results, the smaller generalization gap of thresholding after random projection was observed for ``messy" data: the ones where the classes are mixed up together. In such cases, even simple linear algorithms such as logistic regression or linear SVM can overfit the training samples if the dimension is high and the training set is relatively small. An example of such a dataset might be medical data, where one might not have a large training sample, but the dimension of the data might be very large. In the following section we will see a real dataset of images, that has this property: the classes are mixed up together, the dimension is high and the simple linear algorithms (logistic regression and linear SVM) overfit the training data more than thresholding after random projection with $n=20,000$ projections.

\section{Approximation properties }\label{optimality}

\noindent In this section we show that the error of the classifier we consider converges to the optimal error when the number of projections $n$ and the order of polynomial transformation $k$ converge to infinity. This part corresponds to the first step of the consistency argument from \cite{Geman_bias_variance_dilemma}. We show that the set of classifiers $\mathcal{F}_{n,k}$ is rich enough to approximate the optimal Bayes decision function given mild conditions on the class conditional distributions. The second step of the consistency argument from \cite{Geman_bias_variance_dilemma} is discussed in this work in the previous section, where instead of the VC dimension, we are using the number of projection directions $n$ to uniformly bound the difference between training and testing errors by a term that converges to $0$ as $N$ goes to infinity. These two steps show that this method has a potential to fit the data and generalize well, which results in successful learning.

We approach the topic of optimality from a theoretical perspective, assuming full knowledge of the class conditional distributions that generate the data and optimal separation given by Bayes decision rule. We start with a case where the optimal decision function is linear. In this case using the thresholding after random projection classification method on the original data enough times (i.e., with $n$ large enough) abates the reducible error. If the optimal decision function is not linear, we show that for each $\varepsilon>0$ there exists a polynomial of some degree $k$ for which we can find $n$ such that the reducible error is smaller than $\varepsilon$ with probability as large as desired (Corollary \ref{corollary}).

Let us assume that the data to classify come from two different classes $\omega_1$ and $\omega_2$ that follow probability distributions with densities $\rho_1(\boldsymbol{x}) = \rho(\boldsymbol{x}|\omega_1)$ and $\rho_2(\boldsymbol{x}) = \rho(\boldsymbol{x}|\omega_2)$ respectively. Each class has a certain probability of occurring $P(\omega_1)$ and $P(\omega_2)$, called prior. The mixture of the densities can be expressed in the following way:
\begin{align*}
    \rho(\boldsymbol{x}) = \rho_1(\boldsymbol{x})P(\omega_1) + \rho_2(\boldsymbol{x})P(\omega_2).
\end{align*}
Let $\mathcal{B}$ be the support of the function $\rho(\boldsymbol{x})$.

Bayes Classification Rule chooses a class for $\boldsymbol{x}$ that maximizes $\rho_i(\boldsymbol{x})P(\omega_i)$. The probability of an overall error for the classifier is minimized by choosing a class for every point according to this Rule. In case it is well defined, this minimal error is called Bayes Error ($BE$) and can be expressed in the following way (see \cite{DudaHartStork01} for example):
\begin{align*}
  BE = \int_{\mathcal{B}}\min\{\rho_1(\boldsymbol{x})P(\omega_1), \rho_2(\boldsymbol{x})P(\omega_2)\}d\boldsymbol{x}.  
\end{align*}
The overall error $R_{\text{popul}}(f)$ that any classifier makes is therefore no smaller than Bayes Error and their difference is called the reducible error:
\begin{align*}
\text{Reducible error}(f) = R_{\text{popul}}(f)-BE.
\end{align*}

In most of the following proofs we need to assume that $\mathcal{B}$ is a compact subset of ${\mathbb R}^d$. For many distributions (e.g.~mixture of Gaussians) this assumption does not apply. However, for distributions that have a finite first moment it is possible to identify a compact region so that the decision that we make outside of that region has an arbitrarily small influence on the overall error. For a fixed $\varepsilon>0$ and distribution $\rho(\boldsymbol{x})$, let us define a compact set $\mathcal{B}(\varepsilon)$ with the following property: $\int_{\mathcal{B}(\varepsilon)^{C}}\rho(\boldsymbol{x})d\boldsymbol{x}<\varepsilon$. This set can be obtained in the following way:
\begin{align}\label{cropped_compact}
    \mathcal{B}(\varepsilon):=\{\boldsymbol{x}\in \mathbb{R}^d: \|\boldsymbol{x}-\int_{\mathbb{R}^d}\boldsymbol{x}\rho(\boldsymbol{x})d\boldsymbol{x}\|_2\leq t(\varepsilon)\},
\end{align}
where $t(\varepsilon) = \arginf_{s\in \mathbb{R}}\int_{\|\boldsymbol{x}-E\boldsymbol{x}\|_2>s}\rho(\boldsymbol{x})d\boldsymbol{x}<\varepsilon.$ We know that for every $\varepsilon>0$, $t(\varepsilon)>-\infty$, because $\int_{\boldsymbol{x}\in\mathbb{R}^{d}}\rho(\boldsymbol{x})=1$ and therefore $\lim_{t \rightarrow \infty}\int_{\|\bs{x}-E\bs{x}\|_2>t}\rho(\boldsymbol{x})d\boldsymbol{x}=0.$ If we additionally assume that distribution of $\boldsymbol{X}$ has a finite variance that is a multiple of an identity matrix ($\text{Cov}(\boldsymbol{X})=\sigma^2\mathbb{I}$), we can use Chebyshev inequality to find $t(\varepsilon)$. Since $   P\left(\|\boldsymbol{X} - E\boldsymbol{X}\|_2\geq k\sigma\right) \leq \frac{d}{k^2}$, then by choosing $t(\varepsilon)=\sigma\sqrt{\frac{d}{\varepsilon}}$ we get that 
\begin{align*}
    P\left(\|\boldsymbol{X} - E\boldsymbol{X}\|_2>t(\varepsilon)\right)\leq \varepsilon.
\end{align*}
\color{black}
 For the method of thresholding after random projection, we generate directions for projections at random following Uniform distribution on a unit hypersphere. This distribution is such that every open set on the hypersphere has a non-zero probability: if $u$ is an open set in $S^{d-1}$ and $f_a$ denotes the distribution of the random vectors for projection, then: $p_u = \int_{u}f_a(\boldsymbol{a})d\boldsymbol{a}>0$.

\begin{theorem} \label{linear_case}
If the optimal decision function given by Bayes rule is linear, then for $k=1$ (i.e., using the original feature space coordinates without extension) the reducible error of the method of thresholding after random projection converges to $0$ in probability as $n$ goes to infinity.
\end{theorem}
\begin{proof}
Let us denote the unit normal vector to the optimal separation hyperplane given by Bayes' rule as $\boldsymbol{a}$. If the random vector drawn $\widehat{\boldsymbol{a}}$ is equal to $\boldsymbol{a}$ then the method of thresholding after random projection will be optimal and the error will be equal to Bayes error. Let $u$ be an open neighborhood of $\boldsymbol{a}$ on the unit hypersphere and let $p_{u}$ be the probability that $\widehat{\boldsymbol{a}} \in u$. By assumption $p_{u} \neq 0$. Consider $n$ independent random samples of the vector $\{\boldsymbol{a}^{i}\}_{i=1}^{n}$. The probability that all of these vectors lie outside of $u$ is equal to $(1-p_u)^n$, and thus converges to $0$ as $n$ goes to infinity. Therefore, with $n$ large enough, we can get as close to the optimal separation hyperplane as needed by choosing the vector ${{\boldsymbol{a}}^{(n)}}$ that minimizes the distance to the optimal $\boldsymbol{a}$. The vector chosen for classification converges to the optimal one in probability:
\begin{equation}\label{void}
    \forall \varepsilon>0: \; \lim_{n \rightarrow \infty}P\left(\vert\vert {\boldsymbol{a}}^{(n)} - \boldsymbol{a} \vert\vert_2 > \varepsilon\right) = \lim_{n \rightarrow \infty}(1-p_{u_{\varepsilon}})^n  = 0,
\end{equation}
where 
\begin{align*}
    p_{u_{\varepsilon}} = \int_{\|\boldsymbol{a}-\hat{\boldsymbol{a}}\|_2>\varepsilon} f_a(\boldsymbol{a})d\boldsymbol{a}.
\end{align*}
\noindent Let us fix $\varepsilon>0$ and find $n$ such that the reducible error of the method of thresholding after random projection is smaller than $\varepsilon$. Let us define a compact set $\mathcal{B}_{\varepsilon/2}$ according to \ref{cropped_compact}. The error that we make outside of $\mathcal{B}_{\varepsilon/2}$ will be smaller that $\varepsilon/2$ no matter what our classification is.  The reducible error on $\mathcal{B}_{\varepsilon/2}$ can be expressed as the following integral:
\begin{equation*}
    \frac{1}{2}\int_{\mathcal{A}_n \cap \mathcal{B}_{\varepsilon/2}} \max\{\rho_1(\boldsymbol{x})P(\omega_1),\rho_2(\boldsymbol{x})P(\omega_2)\}-\min\{\rho_1(\boldsymbol{x})P(\omega_1),\rho_2(\boldsymbol{x})P(\omega_2)\}d\boldsymbol{x},
\end{equation*}
where $\mathcal{A}_n$ is the set of points where the classifier does not make an optimal decision given by Bayes rule. The region $\mathcal{A}_n$ is an intersection of two half-spaces (given by the separation hyperplanes). The area of $\mathcal{A}_n$ can be characterized by the dihedral angle between the two hyperplanes, let us denote it by $\alpha_n$. 
\begin{align}\label{reducible_error}
         \frac{1}{2}\int_{\mathcal{A}_n \cap \mathcal{B}_{\varepsilon/2}} \max_i\{\rho_i(\boldsymbol{x})P(\omega_i)\}-\min_i\{\rho_i(\boldsymbol{x})P(\omega_i)\}d\boldsymbol{x} \leq M \int_{\mathcal{A}_n \cap \mathcal{B}_{\varepsilon/2}}d\boldsymbol{x},
\end{align}
    where $M$ is a finite upper bound for continuous function $\max\{\rho_1P(\omega_1),\rho_2P(\omega_2)\}$ on a compact set $\mathcal{B}_{\varepsilon/2}$. The integral $\int_{\mathcal{A}_n \cap \mathcal{B}_{\varepsilon/2}}d\boldsymbol{x}$ is the volume of an intersection of the set $\mathcal{A}_n$ with the set $\mathcal{B}_{\varepsilon/2}$. We can provide an upper bound for this volume by considering a volume of two hyperspherical sectors with colatitude angle equal to $\alpha_n$ in a hypersphere with radius $B = \diam(\mathcal{B}_{\varepsilon/2})$. Due to \cite{area_hyper_cap_Li} the hypersector has the following volume:
    \begin{align*}
        V_d^{\text{sector}}(B) = \dfrac{\pi^{d/2}}{2\Gamma\left(\frac{d}{2}+1\right)}B^d I_{\sin^2(\alpha_n)}\left(\frac{d-1}{2}, \frac{1}{2}\right),
    \end{align*}
    where $I$ is a regularized incomplete beta function:
    \begin{align*}
        I_{\sin^2(\alpha_n)}\left(\frac{d-1}{2}, \frac{1}{2}\right) = \dfrac{1}{B\left(\frac{d-1}{2}, \frac{1}{2}\right)} \int_0^{\sin^2(\alpha_n)}u^{\frac{d-3}{2}}(1-u)^{-\frac{1}{2}}du,
    \end{align*}
    which converges to $0$ as $\alpha_n$ converges to $0$. The volume of $\mathcal{A}_n\cap\mathcal{B}_{\varepsilon/2}$ also converges to $0$, since:
    \begin{align}\label{volumebysector}
    \int_{\mathcal{A}_n\cap\mathcal{B}_{\varepsilon/2}}d\boldsymbol{x} \leq 2V_d^{\text{sector}}(B).
    \end{align}
\noindent Since ${\boldsymbol{a}}^{(n)}$ approaches $\boldsymbol{a}$ in probability, the dihedral angle $\alpha_n$ approaches $0$ in probability. Thresholding is a continuous function of the projection direction. Therefore if  $\lim_{n\rightarrow\infty}\widehat{\boldsymbol{a}}_n = \boldsymbol{a}$ in probability, the thresholds that correspond to each projection converge to the threshold that corresponds to the optimal projection. Therefore, the separating hyperplane chosen by the method is not parallel to the optimal separating hyperplane. Due to (\ref{volumebysector})
    \begin{align*}
        \forall \varepsilon>0 \;\; \lim_{n \rightarrow \infty}P\left(M\int_{\mathcal{A}_n\cup\mathcal{B}_{\varepsilon/2}}d\boldsymbol{x} > \varepsilon/2\right)=0
    \end{align*}
and using (\ref{reducible_error}) we have that
\begin{align*}
    P\left(\text{reducible error}>\varepsilon\right) \leq P\left(\int_{\mathcal{B}_{\varepsilon/2}^C}\rho(\boldsymbol{x})d\boldsymbol{x}\right)+P\left(M\int_{\mathcal{A}_n\cup\mathcal{B}_{\varepsilon/2}}d\boldsymbol{x}>\varepsilon/2\right),
\end{align*}
where first probability is equal to $0$ due to construction of the compact set $\mathcal{B}_{\varepsilon/2}$. Therefore the reducible error of the method converges to $0$ in probability. 
\end{proof}

If the optimal function for classification according to Bayes rule is not linear, then in order to achieve an arbitrarily small reducible error, we extend the feature space using monomials of order up to some $k>1$. The result of the extension is that we decide the class based on the sign of a polynomial of degree $k$ with random coefficients. In the following section we study the approximation properties of such polynomials. One of the results (Theorem \ref{polynomial_case}) was inspired by the Universal Approximation Theorem from \cite{universal_theorem_cybenko}. Theorem \ref{polynomial_case} shows that with large probability the sign of a Boolean classification function can be approximated on an arbitrary large subset of the support by a polynomial with random coefficients.

\begin{theorem}\label{continuous_by_random_polynom}
Given $\varepsilon>0$ and a continuous function $g$ on a compact set $\mathcal{B}$ there exists $k$, such that the probability that amongst $n$ randomly generated polynomials of degree $k$ there exists one with supremum distance from $g$ smaller than $\varepsilon$ converges to $1$ as $n$ converges to infinity:
\begin{align*}
    \forall_{\varepsilon>0} \;\; \forall_{g \in \mathcal{C}(\mathcal{B})} \;\; \exists_{k \in \mathbb{N}} \;\; P\left(\min_{i \in \{1,...,n\}}\|p^{(i)}_k-g\|_{\infty}\leq\varepsilon \right) \xrightarrow{n\rightarrow \infty} 1
\end{align*}
\end{theorem}
\begin{proof}
We can approximate continuous function $g$ by a polynomial $p$. Indeed, according to Stone-Weierstrass theorem since $g$ is a continuous real-valued function defined on a closed and bounded set $\mathcal{B} \subset \mathbb{R}^d$ for each $\widehat{\varepsilon}>0$ there exists a polynomial $p(\boldsymbol{x})$, such that:
\begin{align*}
    \sup_{\boldsymbol{x} \in \mathcal{B}}|g(\boldsymbol{x})-p(\boldsymbol{x})|<\widehat{\varepsilon}.
\end{align*}
The next question is whether we can approximate such a polynomial with a polynomial whose coefficients are chosen at random with a sufficiently high number or random draws. The optimal polynomial $p(\boldsymbol{x})$ is of certain degree $k$. Let us consider a polynomial transform of degree $k$ of $\boldsymbol{x}$ using a mapping $\phi_k: \mathbb{R}^d \rightarrow \mathbb{R}^{\tilde{d}}$:
  \begin{align*}
  \phi_k(\boldsymbol{x}) &= (1, x_1, x_2,..., x_i^{a_i}x_j^{a_j}x_m^{a_m},..., x_d^k), \\
  \left(\phi_k(\boldsymbol{x})\right)_{J} &= \prod_{i\in J}x_i^{a_i}, \;\;\; \sum_{i\in J}a_i\leq k, \;\;\; J \subseteq \{1, ...,d\}.
  \end{align*}
  The dimension of the new space of features is $\tilde{d}$, which depends on the original dimension $d$ and the degree $k$ of the optimal polynomial $p(\boldsymbol{x})$. Consider randomly generated coefficients 
  \begin{align*}
    \boldsymbol{a}^{(i)} \sim \text{Unif}\left(S^{\tilde{d}-1}\right)  \text{, where }\tilde{d}=\binom{d+k}{k}.
  \end{align*}
 Let us construct a polynomial as a dot product between the generated coefficients and a polynomial transformation of order $k$:
  \begin{align}\label{recepy}
      p^{(i)}(\boldsymbol{x}) = \boldsymbol{a}^{(i)} \cdot \phi_k(\boldsymbol{x}).
  \end{align}
Our target polynomial is $p(\boldsymbol{x})$, can we get as close as possible to it by choosing one of $n$ randomly generated polynomials? In other words, is it true that with large probability
\begin{align*}
    \forall_{\varepsilon>0} \;\; \exists_{n} \text{ and } \exists_{ {p^{(\hat{n})}}(\boldsymbol{x}) \in \{p^{(1)}(\boldsymbol{x}),\; ..., \;p^{(n)}(\boldsymbol{x})\}}:\; \sup_{\boldsymbol{x} \in \mathcal{B}} |p(\boldsymbol{x})-p^{(\hat{n})}(\boldsymbol{x}))|\leq \varepsilon?
\end{align*}
One restriction on random polynomials following formula (\ref{recepy}) is that $\|\boldsymbol{a}^{(i)}\|=1$. The norm of the coefficients ${\boldsymbol{a}}$ in the optimal polynomial $p(\boldsymbol{x})$ does not have to be 1, but we can rescale the coefficients by $\|{\boldsymbol{a}}\|$ and get a rescaled polynomial $\widetilde{p}(\boldsymbol{x})$:
\begin{align*}
    \widetilde{p}(\boldsymbol{x}) =  \dfrac{p(\boldsymbol{x})}{\|{\boldsymbol{a}}\|} = \dfrac{{a}_0}{\|{\boldsymbol{a}}\|} +  \dfrac{{a_1}}{\|{\boldsymbol{a}}\|}x_1+...+\dfrac{{a_{\tilde{d}}}}{\|{\boldsymbol{a}}\|}x_d^k. 
\end{align*}
Let us denote the rescaled coefficients by $\widetilde{\boldsymbol{a}}$. With probability converging to $1$ (see the proof of Theorem \ref{linear_case}), for each $\varepsilon'>0$ we can choose coefficients $\boldsymbol{a}^{(\hat{n})}$ in an open ball around the optimal coefficients: $\left|\left|\boldsymbol{a}^{(\hat{n})}-\widetilde{\boldsymbol{a}}\right|\right|_{\infty} < \varepsilon'$. Then, for every $\boldsymbol{x} \in \mathcal{B}$:
\begin{align} \label{polyn_bound}
\begin{split}
    |p^{(\hat{n})}(\boldsymbol{x})-\widetilde{p}(\boldsymbol{x})| = \left|\left(a_0^{(\hat{n})} - \widetilde{a}_{0}\right) + \left(a_1^{(\hat{n})} - \widetilde{a}_{1}\right)x_1+...\left(a_{\tilde{d}}^{(\hat{n})} - \widetilde{a}_{\tilde{d}}\right)x_d^k\right| \leq \\ \leq \left|\left|\boldsymbol{a}^{(\hat{n})}-\widetilde{{\boldsymbol{a}}}\right|\right|_{\infty}\left(1+|x_1|+|x_2|+...+|x_d^k| \right)
    \leq \left|\left|\boldsymbol{a}^{(\hat{n})}-\widetilde{{\boldsymbol{a}}}\right|\right|_{\infty} \|\phi(\boldsymbol{x})\|_1 \leq \\ \leq \left|\left|\boldsymbol{a}^{(\hat{n})}-\widetilde{{\boldsymbol{a}}}\right|\right|_{\infty}  \tilde{d} \; \|\boldsymbol{x}\|_{\star}^{k} \leq \tilde{d}\; B_{\star}^k\; \varepsilon',  
    \end{split}
\end{align}
where $\|\boldsymbol{x}\|_{\star} = \|\boldsymbol{x}\|_1\lor1$ and $B_{\star} = \diam(\mathcal{B})\lor 1$ which implies that    $\|\boldsymbol{x}\|_{\star}\leq B_{\star}$. Therefore if we choose $\varepsilon'$ and $\widehat{\varepsilon}$ so that 
\begin{align*}
    \varepsilon' \leq \dfrac{\varepsilon}{2\tilde{d}B^k_{\star}\|\boldsymbol{a}\|} \; \text{  and  } \; \widehat{\varepsilon} \leq \dfrac{\varepsilon}{2},
\end{align*}
then the following inequality holds:
\begin{align*}
    \sup_{\boldsymbol{x}\in {\mathcal{B}}}\left|\|\boldsymbol{a}\|p^{(\hat{n})}(\boldsymbol{x})-g(\boldsymbol{x})\right| \leq \sup_{\boldsymbol{x}\in {\mathcal{B}}}\left(\left|\|\boldsymbol{a}\|p^{(\hat{n})}(\boldsymbol{x})-{p}(\boldsymbol{x})\right|\right)+\sup_{\boldsymbol{x}\in \mathcal{B}}\left(\left|p(\boldsymbol{x})-g(\boldsymbol{x})\right|\right) \leq \varepsilon.
\end{align*}
\end{proof}

\begin{theorem}\label{polynomial_case} 
Given $\varepsilon>0$ and $\eta>0$ and a Boolean function $c$ that splits the compact set $\mathcal{B}$ into two measurable sets $\mathcal{B}_1$ and $\mathcal{B}_2$, there exists $\widetilde{\mathcal{B}} \subset \mathcal{B}$ and $k$, such that $\mu(\mathcal{B}\setminus\widetilde{\mathcal{B}})\leq \eta$ and the probability that amongst $n$ randomly generated polynomials of degree $k$ there exists one with a supremum distance from $c$ smaller than $\varepsilon$ on $\widetilde{\mathcal{B}}$ converges to $1$ as $n$ converges to infinity:
\begin{align*}
    \forall_{\varepsilon>0} \;\; \forall_{\eta>0} \;\; \forall_{c \in \mathcal{M}(\mathcal{B})} \;\; \exists_{\widetilde{\mathcal{B}}\subset \mathcal{B}}\;\;\exists_{k \in \mathbb{N}}: \;\;\mu(\mathcal{B}\setminus \widetilde{\mathcal{B}})\leq \eta, \;\; P\left(\min_{i \in \{1,...,n\}}\|p^{(i)}_k-c\|_{\infty}\leq\varepsilon \right) \xrightarrow{n\rightarrow \infty} 1,
\end{align*}
where 
$\mathcal{M}(\mathcal{B})$ is a set of all Boolean functions that split compact set $\mathcal{B}$ into two measurable sets $\mathcal{B}_{1}$ and $\mathcal{B}_{2}$. 
\end{theorem}

\begin{proof}
The structure of the proof is similar to the one from \cite{universal_theorem_cybenko}. First, we approximate the Boolean classification function using a continuous function, then we use the result of Theorem \ref{continuous_by_random_polynom} and approximate the continuous function by a polynomial with random coefficients. \par
We seek to approximate a function $c(\boldsymbol{x})$ that is equal to $1$ if $\boldsymbol{x} \in \mathcal{B}_1$ and $-1$ if $\boldsymbol{x} \in \mathcal{B}_2$. By Lusin's theorem, on a measure space $(\mathbb{R}^d, \mathcal{B}(\mathbb{R}^d), \mu)$, where $\mu$ is a Lebesgue measure, for an arbitrary large subset of $\mathcal{B}$ there exists a continuous function $g$ that is equal to $c$ on that set. That is, for each $\eta>0$ there exists a set $\mathcal{\tilde{B}} \subset \mathcal{B}$, such that $\mu(\mathcal{B} \setminus \mathcal{\tilde{B}})<\eta$ and:
\begin{align*}
    c(\boldsymbol{x}) = g(\boldsymbol{x}), \text{ for all } \boldsymbol{x} \in \mathcal{\tilde{B}}.
\end{align*}
Note, that $\widetilde{\mathcal{B}}$ is compact and Theorem \ref{continuous_by_random_polynom} shows that there exists a polynomial with random coefficients $p^{(\hat{n})}(\boldsymbol{x})$ that is close to the continuous function on $\widetilde{\mathcal{B}}$ in supremum distance, therefore:  
\begin{align*}
     \sup_{\boldsymbol{x}\in \widetilde{\mathcal{B}}}\left|\|\boldsymbol{a}\|p^{(\hat{n})}(\boldsymbol{x})-c(\boldsymbol{x})\right| = \sup_{\boldsymbol{x}\in \widetilde{\mathcal{B}}}\left|\|\boldsymbol{a}\|p^{(\hat{n})}(\boldsymbol{x})-g(\boldsymbol{x})\right|\leq \varepsilon.
\end{align*}
\end{proof}

\begin{corollary}\label{corollary_random_threshold} 
Assuming that Bayes decision rule splits the compact support $\mathcal{B}$ into two measurable sets, the reducible error of choosing a class according to the sign of a polynomial of degree $k$ with random coefficients which has the smallest population error out of $n$ such polynomials converges to $0$ in probability as $k$ and $n$ converge to infinity:
\begin{align*}
    \forall_{\varepsilon>0}\;\; P \left(|R_{\text{popul}}\left(\sign\left(p^{(\hat{n})}\right)\right)-BE|>\varepsilon\right) \xrightarrow{n\rightarrow \infty,\;k \rightarrow \infty} 0.
\end{align*}
\end{corollary}
\begin{proof}
Let us construct a Boolean function $c(\boldsymbol{x})$ according to the Bayes decision rule. The previous theorem states that for any $\eta>0$ there exists a set $\widetilde{\mathcal{B}}$, such that $\widetilde{\mathcal{B}} \subset \mathcal{B}$ and $\mu(\mathcal{B}\setminus\widetilde{\mathcal{B}})< \eta$. And on this set:
\begin{align*}
     \sup_{\boldsymbol{x}\in \widetilde{\mathcal{B}}}\left|\|\boldsymbol{a}\|p^{(\hat{n})}(\boldsymbol{x})-c(\boldsymbol{x})\right| \leq \varepsilon
\end{align*}
for any $\varepsilon>0$ with probability that converges to $1$ as $n$ goes to infinity. If $\varepsilon < 1/2$ then
\begin{align*}
    \sign\left({p^{(\hat{n})}}(\boldsymbol{x})\right) = \sign(\|\boldsymbol{a}\|p^{(\hat{n})}(\boldsymbol{x})) = c(\boldsymbol{x}) \;\; \text{ on } \widetilde{\mathcal{B}}.
\end{align*}
That means that the decision we make according to the sign of the polynomial with random coefficients is the same as the optimal decision on $\widetilde{\mathcal{B}}$. The reducible error is therefore concentrated on the set $\mathcal{B}\setminus\widetilde{\mathcal{B}}$ and can be estimated in the following way:
\begin{align}\label{red_error}
\begin{split}
    \int_{\mathcal{B}}\mathbbm{1}\left(c(\boldsymbol{x})\neq \sign(\widehat{p_n}(\boldsymbol{x}))\right)\left(\max_i\{\rho_i(\boldsymbol{x})P(\omega_i)\}-\min_i\{\rho_i(\boldsymbol{x})P(\omega_i)\}\right)d\boldsymbol{x} = \\ =\int_{\mathcal{B} \setminus\mathcal{\widetilde{B}}}\left(\max_i\{\rho_i(\boldsymbol{x})P(\omega_i)\}-\min_i\{\rho_i(\boldsymbol{x})P(\omega_i)\}\right)d\boldsymbol{x} \leq \\
    \leq \int_{\mathcal{B} \setminus \mathcal{\widetilde{B}}}\max_i\{\rho_i(\boldsymbol{x})P(\omega_i)\}d\boldsymbol{x}\leq M\mu(\mathcal{B} \setminus \mathcal{\widetilde{B}}) \leq M \eta,
    \end{split}
\end{align}
Where $M=\sup_{\boldsymbol{x} \in \mathcal{B}}\max_i\{\rho_i(\boldsymbol{x})P(\omega_i)\}d\boldsymbol{x}$. If we want to achieve a reducible error of $\tilde{\varepsilon}$, we can therefore choose $\eta$ as $\tilde{\varepsilon}/M$.

As a consequence, for any $\tilde{\varepsilon}>0$ we can find $\eta= \eta(\tilde{\varepsilon})$, $k = k(\eta)$ and $n = n(k,d, \delta)$, so that with probability $1-\delta$ the reducible error of the classification based on the sign of the polynomial with random coefficients that is chosen out of $n$ tries is smaller than $\tilde{\varepsilon}$. That means that the reducible error of such classification converges to zero in probability. 
\end{proof}
\begin{corollary}\label{corollary}
Under the assumption that Bayes decision rule splits the support into two measurable sets, the reducible error of the method of thresholding after random projection converges to $0$ in probability as $k$ and $n$ converge to infinity.
\end{corollary}
\begin{proof}
According to \ref{corollary_random_threshold} on any compact set we have
\begin{align*}
\forall_{\varepsilon>0}\;\; \forall_{\delta>0}\;\;\exists_{n,k}: P\left(|R_{\text{popul}}\left(\sign\left(p^{(\hat{n})}\right)\right)-BE|>\varepsilon\right) \leq \delta/2.
\end{align*}
Let us choose a compact set $\mathcal{B}_{\delta/2}$ according to (\ref{cropped_compact}). Since Bayes decision rule splits the support into two measurable sets and $\mathcal{B}_{\delta/2}$ is measurable due to its construction, Bayes decision rule splits $\mathcal{B}_{\delta/2}$ into two measurable sets as well (intersection of two measurable sets is measurable). Then 
\begin{align*}
P\left(|R_{\text{popul}}\left(\sign\left(p^{(\hat{n})}\right)\right)-BE|>\varepsilon\right) \leq P\left(|R_{\text{popul}}^{\mathcal{B}_{\delta/2}}\left(\sign\left(p^{(\hat{n})}\right)\right)-BE^{\mathcal{B}_{\delta/2}}|>\varepsilon\right)+ \\ + P\left(\mathcal{B}_{\delta/2}^C\right) \leq \delta/2 + \delta/2.
\end{align*}
where $R_{\text{popul}}^{\mathcal{B}_{\delta/2}}(f)$ is a population error of $f$ restricted to the compact set $\mathcal{B}_{\delta/2}$ and $BE^{\mathcal{B}_{\delta/2}}$ is Bayes error restricted to the compact set $\mathcal{B}_{\delta/2}$. 
The polynomials in \ref{corollary_random_threshold} are of degree $k$ with coefficients $\boldsymbol{a}_i$ picked randomly from a unit hypersphere in $\mathbb{R}^{\tilde{d}}$. This corresponds to projecting the data on a random line and choosing the classification threshold at random.  In the method of thresholding after random projection, all the coefficients are chosen at random, apart from the threshold, which corresponds to the constant term in a polynomial classifier.  Recall that the threshold is optimized to minimize the population error. For each polynomial $p(\boldsymbol{x})$ with random coefficients we can consider a polynomial $p^{\star}(\boldsymbol{x})$ that is used by the method of thresholding after random projection. Polynomial $p^{\star}(\boldsymbol{x})$ has exactly the same coefficients as $p(\boldsymbol{x})$ apart from the absolute coefficient that is:
\begin{align*}
a_0^{\star} = \argmin_{a \in \mathbb{R}} R_{\text{popul}}(\sign(p_i(\boldsymbol{x})-a_0+a)).  
\end{align*}
Therefore
\begin{align*}
    R_{\text{popul}}(\sign(p^{\star (\hat{n})})) \leq R_{\text{popul}}(\sign(p^{(\hat{n})})) 
\end{align*}
and 
\begin{align*}
    P\left(|R_{\text{popul}}(\sign(p^{\star (\hat{n})}))-BE|> \varepsilon\right) \leq     P\left(|R_{\text{popul}}(\sign(p^{ (\hat{n})}))-BE|> \varepsilon\right)
\end{align*}
and 
\begin{align*}
    \forall_{\varepsilon>0}\;\;\forall_{\delta>0}\;\;\exists_{n,k}: P\left(|R_{\text{popul}}(\sign(p^{\star (\hat{n})}))-BE|> \varepsilon\right) \leq \delta,
\end{align*}
which means that the reducible error of thresholding after random projection converges to $0$ when $n$ and $k$ converge to infinity.
\end{proof}

In the proof of Theorem \ref{polynomial_case} we show, that for each reducible error $\varepsilon>0$ there exists a degree $k = k(\varepsilon)$ and an optimal polynomial $p(\boldsymbol{x})$ of degree $k$, such that if we decide based on the sign of this polynomial, the reducible error will be smaller than $\varepsilon$. In order to get close to this optimal polynomial with random ones, one needs to have a certain number of projections to choose from. That number depends on the degree $k$, the dimensionality of the original data $d$, and the tolerance $\delta$: $n=n(k,d,\delta)$. The following lemma estimates how large $n$ has to be to guarantee that with probability $1-\delta$ the reducible error is smaller than $\varepsilon$, assuming that $k(\varepsilon)$ is big enough. 
\begin{theorem}\label{howbign}
Let $k=k(\varepsilon)$ be a feature extension degree that is high enough, in principle, to obtain a reducible error smaller than $\varepsilon$. 
In order to achieve such a small error with probability at least $1-\delta$, the number of projections $n$ has to be no smaller than \begin{align}\label{howlargen}
     n \geq \dfrac{\ln(\delta)}{\ln\left(1-I_{\sin^2\left(4\arcsin\left(\dfrac{1}{8\tilde{d}}\right)\right)}\left(\dfrac{\tilde{d}-1}{2}, \dfrac{1}{2}\right)\right)},
\end{align}
where $ \tilde{d} = \binom{d+k}{k}$ and $I_{\phi}(a,b)$ is a regularized incomplete beta function.
\end{theorem}
\begin{proof}

The following analysis uses formulas for the area of hyperspherical cap from \cite{area_hyper_cap_Li}. Since $\boldsymbol{a}$ is chosen from the unit sphere in $\tilde{d}$ dimensions uniformly at random, we are interested in the probability of missing a desirable set (two hyperspherical caps) that is close enough to the optimal vector of coefficients $\tilde{\boldsymbol{a}}$. The area of a cap given by an angle $\phi$ is 
\begin{align*}
    A_{\tilde{d}}^{\text{cap}} = \dfrac{2\pi^{\frac{\tilde{d}-1}{2}}}{\Gamma\left(\frac{\tilde{d}-1}{2}\right)}\int_{0}^{\phi}\sin^{\tilde{d}-2}(\theta)d\theta.
\end{align*}
The probability that we will end up in the cap is thus given by 
\begin{align*}
    \dfrac{A_{\tilde{d}}^{\text{cap}}}{A_{\tilde{d}}} = \dfrac{1}{2}I_{\sin^2{\phi}}\left(\dfrac{\tilde{d}-1}{2}, \dfrac{1}{2}\right) = \dfrac{B\left(\sin^2{\phi};\dfrac{\tilde{d}-1}{2},\dfrac{1}{2}\right)}{2 B\left(\dfrac{\tilde{d}-1}{2},\dfrac{1}{2}\right)},
\end{align*}
where $I_{\sin^2{\phi}}$ is a regularized incomplete beta function - a cumulative distribution function of the beta distribution. The projection direction given by the optimal vector $\boldsymbol{\tilde{a}}$  is the same as the one given by $-\boldsymbol{\tilde{a}}$. Therefore a good reducible error will not be achieved if we miss two caps that are symmetric with respect to the origin. The probability of missing the neighborhoods is given by:
\begin{align*}
     1-I_{\sin^2\phi}\left(\dfrac{\tilde{d}-1}{2}, \dfrac{1}{2}\right).
\end{align*}
We want to find $n$ such that for a given $\delta$ the following inequality holds:
\begin{align*}
      P(\text{missing the caps after $n$ tries}) = \left(1-I_{\sin^2\phi}\left(\dfrac{\tilde{d}-1}{2}, \dfrac{1}{2}\right)\right)^{n} \leq \delta,
\end{align*}
which is equivalent to 
\begin{align*}
    n \geq \dfrac{\ln(\delta)}{\ln\left(1-I_{\sin^2\phi}\left(\dfrac{\tilde{d}-1}{2}, \dfrac{1}{2}\right)\right)}.
\end{align*}
The angle $\phi$ has to be such, that the supremum distance between the chosen vector $\boldsymbol{a}^{(n)}$ and the optimal projection direction $\boldsymbol{\tilde{a}}$ is bounded from above. Due to Corollary \ref{corollary_random_threshold} and Theorem \ref{continuous_by_random_polynom} we need
\begin{align}\label{first_bound}
    \|\boldsymbol{a}^{(n)}-\boldsymbol{\tilde{a}}\|_{\infty} \leq \dfrac{1}{4\tilde{d}\widetilde{B}^k_{\star}\|\boldsymbol{a}\|},
\end{align}
where $\widetilde{B}=\diam({\widetilde{\mathcal{B}}})$ and is bounded above by $B=\diam(\mathcal{B})$. Therefore we can replace (\ref{first_bound}) with the following condition:
\begin{align*}
       \|\boldsymbol{a}^{(n)}-\boldsymbol{\tilde{a}}\|_{\infty} \leq \dfrac{1}{4\tilde{d}{B}^k_{\star}\|\boldsymbol{a}\|}.
\end{align*}
Since the supremum distance is bounded above by $L_2$ distance, we can impose conditions on $\|\boldsymbol{a}^{(n)}-\boldsymbol{\tilde{a}}\|_2 $. The connection between the angle $\phi$ and distance between the vector at the border of the cap and the optimal vector $\widetilde{\boldsymbol{a}}$ is the following: 
\begin{align*}
    \sin\left(\frac{\phi}{4}\right) = \dfrac{\|\boldsymbol{a}^{(n)}-\boldsymbol{\tilde{a}}\|_2}{2},
\end{align*}
therefore $\phi$ must satisfy
\begin{align*}
    \phi \leq 4\arcsin\left({\dfrac{1}{8\tilde{d}B_{\star}^k\|\boldsymbol{a}\|}}\right).
\end{align*}
As a consequence we need $n$ bigger than
\begin{align*} 
        n \geq \dfrac{\ln(\delta)}{\ln\left(1-I_{\sin^2\left(4\arcsin\left(\dfrac{1}{8\|{\boldsymbol{a}}\|{\tilde{d}}B^{k}_{\star}}\right)\right)}\left(\dfrac{\tilde{d}-1}{2}, \dfrac{1}{2}\right)\right)}.
\end{align*}
Assuming that $\|{\boldsymbol{a}}\| = 1$ and $B=1$ we get:
\begin{align*}
    n \geq \dfrac{\ln(\delta)}{\ln\left(1-I_{\sin^2\left(4\arcsin\left(\dfrac{1}{8{\tilde{d}}}\right)\right)}\left(\dfrac{\tilde{d}-1}{2}, \dfrac{1}{2}\right)\right)}.
\end{align*}
\end{proof}

\section{Restricted number of projections needed for a mixture of two Gaussians} \label{restricted_number}
In general, according to Formula \ref{howlargen}, the number of projections $n$ needed for achieving a high classification accuracy could be very large,
which makes the  generalization bounds of Theorem \ref{our_chaining_result} unattractive.
However, for certain classification problems, $n$ might be quite small. An example of this is when the classes are distributed following two far away Gaussian distributions, as discussed in \cite{boutin2024}.

The reason for this can be explained using the Rashomon ratio. To be specific, consider the  following data generating distribution:
\begin{align*}
\bs{X}\mid(Y=r) \sim \mathcal{N}_d(\bs{\mu}_r, \sigma^2\mathbb{I}),
\end{align*}
 where $r \in \{1,2\}$, $\mathbb{I}$ is an identity matrix of dimensionality $d$ and $P(Y=r)=1/2$ for both $r$. 
As discussed in \cite{coupkova2024rashomon}, the Rashomon ratio of such a mixture for $d=1$, $\sigma=1$, $\bs{\mu}_1=-2.5$ and $\bs{\mu}_2=2.5$ is approximately $58\%$. 
Such a high Rashomon ratio means that, with a high probability, the number of functions $n$ that has to be drawn in order to achieve a classification with high accuracy is very small. 
In particular, with probability at least $0.96$, only $7$ functions are sufficient 
to achieve a training error of the method of thresholding after random projection that is less than
 \begin{align*}
     \inf_{f \in \mathcal{F}}R_{\text{train}}(f) \leq \inf_{f \in \mc{F}_{\text{afine}}}R_{\text{popul}}(f)+0.05+\sqrt{\frac{1}{2N}\ln\left(\frac{1}{0.01}\right)},
\end{align*}
where $\inf_{f \in \mc{F}_{\text{affine}}}R_{\text{popul}}(f)$ is the Bayes error for an antipodal mixture of Gaussian distributions, which is equal to $0$ for affine classifiers.
These ideas are explored and extended in  \cite{coupkova2024rashomon}.

\bibliographystyle{plain}
\bibliography{bibliography.bib}
\end{document}